\documentclass{llncs}

\newcommand{\systemSZ}{{system~W}\xspace}
\newcommand{\systemSZgross}{{System~W}\xspace}
\newcommand{\namePS}{preferred structure\xspace}

\newcommand{\fctOPname}{{\ensuremath{\mathit{OP}}}}
\newcommand{\fctOP}[1]{\ensuremath{\fctOPname(#1)}} 
\newcommand{\fctOPvector}[1]{\ensuremath{(\R_0,\ldots,\R_k)}}

\newcommand{\cspRArg}[1]{\ensuremath{\mathit{CR}(#1)}}
\newcommand{\leSZ}{\ensuremath{<^{\sf{w}}_{\R}}}
\newcommand{\leSZarg}[1]{\ensuremath{<^{\sf{w}}_{#1}}}
\newcommand{\inferSZ}{\ensuremath{\nmableit^{\sf{w}}_{\R}}}
\newcommand{\notinferSZ}{\ensuremath{\notnmableit^{\sf{w}}_{\R}}}

\newcommand{\inferSZArg}[1]{\ensuremath{\nmableit^{\sf{w}}_{#1}}}
\newcommand{\inferSZarg}[1]{\ensuremath{\inferSZArg{#1}}}
\newcommand{\falsWname}{\ensuremath{\mathit{\xi}}}
\newcommand{\falsW}{\ensuremath{\falsWname}}
\newcommand{\verify}{\ensuremath{\sf{v}}}
\newcommand{\falsif}{\ensuremath{\sf{f}}}
\newcommand{\neutra}{\ensuremath{\mathit{-}}}
\newcommand{\ol}[1]{\overline{#1}}
\newcommand{\cspR}{\ensuremath{\mathit{CR}(\R)}}
\newcommand{\solutionsR}{\ensuremath{\solutionsOf{\cspR}}}
\newcommand{\solutionsOf}[1]{\ensuremath{\mathit{Sol}(#1)}}
\newcommand{\akzeptiert}{\ensuremath{\models}}
\newcommand{\induzierteOCF}[1]{\ensuremath{\kappa_{#1}}}
\newcommand{\R}{\ensuremath{\mathcal R}}
\newcommand{\notB}{\overline{B}}
\newcommand{\kappaiminus}[1]{\eta_{#1}}
\newcommand{\kappaiminusVektor}{\ensuremath{\vv{\eta}}}

\makeatletter
\newcommand*{\centernot}{\mathpalette\@centernot
}
\def\@centernot#1#2{\mathrel{\rlap{\settowidth\dimen@{$\m@th#1{#2}$}\kern.5\dimen@
			\settowidth\dimen@{$\m@th#1=$}\kern-.5\dimen@
			$\m@th#1\not$}{#2}}}
\makeatother
\DeclareRobustCommand\nmableitSymb{\mathrel|\mkern-.5mu\joinrel\sim} \newcommand{\nmableit}{\ensuremath{\mbox{$\,\nmableitSymb\,$}}} \newcommand{\notnmableit}{\ensuremath{\mbox{$\,\centernot\nmableitSymb\,$}}} 

\newcommand{\satzCL}[2]{\ensuremath{(#1|#2)}}
\newcommand{\Vocabulary}{\ensuremath{\Sigma}}
\newcommand{\symbolForRbirds}{\ensuremath{\mathit{bird}}}
\newcommand{\Rbirds}{\ensuremath{{\R_{\symbolForRbirds}}}}
\newcommand{\RbirdsStern}{\ensuremath{\R_{\symbolForRbirds}^{\ast}}}
 
\usepackage{pgf,tikz}
\usetikzlibrary{arrows,positioning,shapes}
\usepackage{tikz-qtree}

\usepackage{latexsym}
\usepackage{amssymb}
\usepackage{amsmath}
\usepackage{stmaryrd}
\usepackage{esvect}
\usepackage{braket}
\usepackage{xcolor}
\usepackage{xspace}
\usepackage{graphicx}
\usepackage{boxedminipage}
\usepackage{booktabs}
\usepackage{url}

\usepackage{listings}

\newcommand{\MEINqed}{\hfill\qed}

\newtheorem{Observation}{Observation}

\title{Nonmonotonic Inferences with\\ Qualitative Conditionals based on\\ Preferred Structures on Worlds}
\author{Christian Komo \and Christoph Beierle}

\institute{
Dept. of Computer Science, FernUniversit\"{a}t in Hagen,
58084 Hagen, Germany
}

\begin{document}
\maketitle
\begin{abstract}
A conditional knowledge base \R\ is a set of conditionals of the form "If A,
the usually B". Using structural information derived from the conditionals
in \R, we introduce the preferred structure relation on worlds.
The preferred structure relation is 
the core ingredient of a new inference relation called system~W inference that
inductively completes the knowledge given explicitly in \R. We show that system~W exhibits desirable inference properties like satisfying system P and avoiding, in contrast to e.g. system Z, the drowning problem. It fully captures and strictly extends both system Z and skeptical c-inference.
In contrast to skeptical c-inference, it does not require to solve a complex constraint satisfaction problem, but is as tractable as system~Z.

\end{abstract}

\section{Introduction}
\label{sec_introduction}

In the area of knowledge representation and reasoning, conditionals play a prominent role. Nonmonotonic reasoning investigates qualitative conditionals of the form ``\emph{If A then usually B}". Various semantical approaches
for inferences based on
sets of such conditionals as well as criteria and postulates for evaluating
the obtained inference relations have been proposed
(cf.\ 
\cite{Adams1975,Lewis73,Spohn1988,KrausLehmannMagidor90,Pearl1990,GoldszmidtPearl1991AI,LehmannMagidor92short,BenferhatDuboisPrade92,Makinson94,DuboisPrade1994ConditionalObjects,GoldszmidtPearl1996,BenferhatDuboisPrade99short,KernIsberner2001habil,DuboisPrade2015WhereDoWeStand}).
Among the different semantical models of conditional knowledge bases are
Spohn's ordinal conditional functions (OCFs)~\cite{Spohn1988,Spohn2012book},
also called ranking functions. 
An OCF $\kappa$ assigns a degree of surprise to each world $\omega$, the higher the value $\kappa(\omega)$ assigned to $\omega$, the more surprising $\omega$.
Each $\kappa$ that accepts a set \R\ of conditionals, called a knowledge base, induces a nonmonotonic inference relation that inductively completes the explicit knowledge given in \R.

Two inference relations which are defined based on specific OCFs obtained
from a knowledge base \R\ have received some attention:
system~Z~\cite{Pearl1990,GoldszmidtPearl1996} and c-represen\-tations~\cite{KernIsberner2001habil,KernIsberner2004amai}, or the induced inference relations, respectively, both show excellent inference properties.
System Z is based upon the ranking function $\kappa^Z$, which is the unique Pareto-minimal OCF that accepts \R.
The definition of $\kappa^Z$ crucially relies on the notions of \emph{tolerance}
and of \emph{inclusion-maximal ordered partition} of \R\  obtained via the
tolerance relation~\cite{Pearl1990,GoldszmidtPearl1996}.
Among the OCF models of \R, c-rep\-res\-en\-tations are special models obtained by
assigning an individual impact to each conditional and generating 
the world ranks as the sum of impacts of falsified conditionals
 \cite{KernIsberner2001habil,KernIsberner2004amai}.
While for each consistent \R, the system~Z ranking function $\kappa^Z$ is uniquely determined, there may be many different c-representations of \R.
Skeptical c-inference  \cite{BeierleEichhornKernIsberner2016FoIKSshort,BeierleEichhornKernIsbernerKutsch2018AMAI} is the inference relation obtained by taking all c-representations of \R\ into account.

It is known that system Z and skeptical c-inference both satisfy system P \cite{KrausLehmannMagidor90,GoldszmidtPearl1996,BeierleEichhornKernIsbernerKutsch2018AMAI} and other desirable properties. Furthermore, there are system Z inferences that are not obtained by skeptical c-inference, and on the other hand, there are skeptical c-inferences that are not system Z inferences \cite{BeierleEichhornKernIsbernerKutsch2018AMAI}. Another notable difference between system Z and skeptical c-inference is that
the single unique system Z model \cite{Pearl1990} can be computed much easier than skeptical c-inference which involves many models obtained from the solutions of a complex constraint satisfaction problem \cite{BeierleEichhornKernIsbernerKutsch2018AMAI}.
In recently published work \cite{KomoBeierle2020ISAIM}, we showed that 
the exponential lower bound  $  2^{n-1}$ is needed as possible impact factor for
c-representations to fully realize skeptical c-inference, supporting the observation that skeptical c-inference is less tractable than system~Z inference
(cf.\ \cite{GoldszmidtPearl1996,BeierleEichhornKernIsbernerKutsch2018AMAI}).

Inspired by our findings in  \cite{KomoBeierle2020ISAIM}, here
we develop the \emph{\namePS} relation
on worlds 
and propose the new nonomonotonic \emph{\systemSZ inference} based on it.
The main contributions of this paper are:

\begin{itemize}
\item We introduce the \emph{\namePS relation \leSZ} on worlds based on the notions of tolerance and verification/falsification behavior of a knowledge base \R.

\item By employing \leSZ, we develop a new inference relation, called \emph{\systemSZ inference}, which is as tractable as  system Z.

\item  We prove that \systemSZ inference captures and strictly extends both system~Z inference and skeptical c-inference.
  
\item We show that \systemSZ inference exhibits desirable inference properties like satisfying the axioms of system P and avoiding the drowning problem.

\end{itemize}

The rest of the paper is organized as follows. After briefly recalling the required in Section~\ref{sec_preliminaries}, we introduce  the \namePS on worlds and prove several of its properties in Section~\ref{sec_sz_structure}.
In Section~\ref{sec_system_sz}, we give the formal definition of \systemSZ, illustrate it with various examples and show its main properties. 
In Section~\ref{sec_conclusions}, we conclude and point out future work.

\section{Conditional logic, system Z, and c-Representations}\label{sec_preliminaries}

Let $\Sigma=\{v_1,...,v_m\}$ be a propositional alphabet.
A \emph{literal} is the positive ($v_i$) or negated ($\ol{v_i}$) form of a propositional variable,
$\dot v_i$ stands for either $v_i$ or $\ol{v_i}$.
From these we obtain the propositional language $\mathcal{L}$ as the set of formulas of $\Sigma$ closed under negation $\neg$, conjunction $\wedge$, and disjunction $\vee$.
For shorter formulas, we abbreviate conjunction by juxtaposition (i.e., $AB$ 
stands for
$A\wedge B$), and negation by overlining (i.e., $\ol{A}$ is equivalent to $\neg A$).
Let $\Omega_{\Sigma}$ denote the set of possible worlds over $\mathcal{L}$; $\Omega_{\Sigma}$ will be taken here simply as the  set of all propositional interpretations over $\mathcal{L}$ and can be identified with the set of all complete conjunctions over \Vocabulary; we will often just write $\Omega$ instead of $\Omega_{\Sigma}$.
For $\omega \in \Omega$, $\omega \models A$ means that the propositional formula $A \in \mathcal{L}$ holds in the possible world $\omega$.
With $\Omega_A = \{\omega \in \Omega_{\Sigma} \mid \omega \models A\}$,
we denote the set of all worlds in which $A$ holds.

A \emph{conditional} $(B|A)$ with $A,B\in\mathcal{L}$ encodes the defeasible rule ``if $A$ then normally $B$'' and is a trivalent logical entity with the evaluation 
\cite{deFinetti37origLong,KernIsberner2001habil} 
\begin{equation}\label{eq_finetti}
 \begin{split}
  \llbracket(B|A)\rrbracket_\omega&=\left\{\begin{array}{l@{\quad}l@{\quad}l}
                                            \verify &\mbox{iff\quad} \omega\models AB&\mbox{(verification)} \, , \\
                                            \falsif &\mbox{iff\quad} \omega\models A\ol{B}&\mbox{(falsification)} \, , \\
                                            \neutra &\mbox{iff\quad} \omega\models \ol{A}&\mbox{(not applicable)} \, .
                                           \end{array}\right.
 \end{split}                                       
 \end{equation}

An \emph{ordinal conditional function} (OCF, ranking function)~\cite{Spohn1988,Spohn2012}
is a function $\kappa:\Omega\to\mathbb{N}_0\cup\{\infty\}$ that assigns to each world $\omega\in\Omega$ an implausibility rank $\kappa(\omega)$:
the higher $\kappa(\omega)$, the more surprising $\omega$ is.
OCFs have to satisfy the normalization condition that there has to be a world that is maximally plausible, 
i.e., 
$\kappa^{-1}(0)\neq\emptyset$.
The rank of a formula $A$ is defined by
\(  \kappa(A)=\min\{\kappa(\omega) \mid \omega\models A\}\).
An OCF \(\kappa\) \emph{accepts} a conditional \(\satzCL{B}{A}\), denoted by \(\kappa\akzeptiert\satzCL{B}{A}\), if the verification of the conditional is less surprising than its falsification, i.e.,
\(\kappa\akzeptiert\satzCL{B}{A}\) iff \(\kappa(AB) < \kappa(A\ol{B})\).
This can also be understood as a nonmonotonic inference relation between the premise $A$ and the conclusion $B$: 
Basically, we say that $A$ \emph{$\kappa$-entails $B$}, written $A\nmableit^\kappa B$,
if $\kappa $  accepts 
$(B|A)$; formally this if given by
\begin{equation}\label{eq_kappa_entailment}
A\nmableit^\kappa B \mbox{\, \ \ iff\, \ \ } A \equiv \bot \text { or } \kappa(AB) \ < \  \kappa(A\ol{B}).
\end{equation}
Note that the reason for including the disjunctive condition in \eqref{eq_kappa_entailment} is to ensure that $\nmableit^\kappa$ satisfies supraclassicality, i.e., $A \models B$ implies $A \nmableit^\kappa B$, also for the case $A \equiv \bot$ as it is required, for instance, by 
the reflexivity axiom $A \nmableit A$ of system P \cite{Adams1975,KrausLehmannMagidor90}.
Let us remark that $\kappa$-entailment is based on the total preorder on possible worlds induced by a ranking function and can be expressed equivalently by:
\begin{equation}\label{eq_kappa_entailment_equiv}
A\nmableit^\kappa B \mbox{\, \ \ iff\, \ \ }  \forall \omega' \in \Omega _{A \overline{B} } \; \exists \omega \in \Omega_{A B } \; \kappa(\omega) < \kappa(\omega') \, .
\end{equation}

The acceptance relation 
is extended as usual to a set \R\ of conditionals, called a \emph{knowledge base}, by defining \(\kappa \akzeptiert \R\) iff \(\kappa  \akzeptiert \satzCL{B}{A}\) for all \(\satzCL{B}{A} \in \R\).
This is synonymous to saying that \(\kappa\) is \emph{admissible} with respect to \R\xspace\cite{GoldszmidtPearl1996}, or that  \(\kappa\) is
a \emph{ranking model} of \R.
\(\R\) is \emph{consistent} iff it has a ranking model.

Two inference relations which are defined by specific OCFs obtained
from a knowledge base \R\xspace have received some attention:
system~Z~\cite{Pearl1990} and c-represen\-tations~\cite{KernIsberner2001habil,KernIsberner2004amai}, or the induced inference relations, respectively, both show excellent inference properties.
We recall both approaches briefly.

System Z~\cite{Pearl1990} is based upon the ranking function $\kappa^Z$, which is the unique Pareto-minimal OCF that accepts \R.
The definition of $\kappa^Z$ crucially relies on the notion of \emph{tolerance}.
A conditional $\satzCL{B}{A}$ is \emph{tolerated} by a set of conditionals \R\xspace if there is a world $\omega\in\Omega$ such that $\omega\models AB$ and $\omega\models\bigwedge_{i=1}^n(A_i\Rightarrow B_i)$, i.e., iff \(\omega\) verifies $\satzCL{B}{A}$ and does not falsify any conditional in \(\R\).
For every consistent knowledge base, the notion of tolerance yields
 an ordered partition $(\R_0,...,\R_k)$ of $\R$, where each $\R_i$
is tolerated by $\bigcup_{j=i}^k\R_j$. 
The \emph{inclusion-maximal partition} of $\R$, in the following denoted by
$\fctOP{\R} = \fctOPvector{\R}$, is the
ordered partition
of $\R$ where each $\R_i$ is the (with respect to set inclusion) maximal subset of $\bigcup_{j=i}^k\R_j$ that is tolerated by $\bigcup_{j=i}^k\R_j$. 
 This partitioning is unique due to the maximality and can be computed using the consistency test algorithm given in~\cite{GoldszmidtPearl1996}; for an inconsistent knowledge base \R, $\fctOP{\R}$ does not exist.
 Using $\fctOP{\R} = \fctOPvector{\R}$, 
the system Z ranking function $\kappa^Z$ is defined by
\begin{equation}\label{eq_system_z}
\kappa^Z(\omega) := \begin{cases}
0 \, , & \text{ if  } \omega \textrm{ does not falsify  any conditional $r\in \R$}     , \\
1 + \max_{ \substack {1\leq i \leq n \\ \omega\models A_i \overline{B_i } } } Z(r_i)  , & \text{otherwise}  ,
\end{cases}
\end{equation}
where the function $Z:\mathcal{R} \to \mathbb{N}_0$ is given by $Z(r_i) = j $ if $r_i \in \mathcal{R}_j$.

\begin{definition}[system Z inference, \boldmath{$\nmableit_\R^Z$} \cite{Pearl1990,GoldszmidtPearl1996}]
\label{systemz_consequence_relations}
Let \(\R\) be a knowledge base and let $A$, $B$ be formulas.
We say that
$B$  is a \emph{system Z inference of $A$ in the context of \R},
denoted by $A\nmableit_\R^Z B$, iff $A\nmableit^{\kappa^Z} B$ holds.
\end{definition}

Among the OCF models of \R, c-rep\-res\-en\-tations are special models obtained by
assigning an individual impact to each conditional and generating 
the world ranks as the sum of impacts of falsified conditionals.
For an in-depth introduction to c-representations and their use of the principle of conditional preservation ensured by respecting conditional structures, we refer to \cite{KernIsberner2001habil,KernIsberner2004amai}.
The central definition is the following:

 \begin{definition}[c-representation \cite{KernIsberner2001habil,KernIsberner2004amai}]\label{def:c-representation}
  A \emph{c-representation} of  a knowledge base $\R$ is a ranking function $\kappa_{\kappaiminusVektor}$ constructed from $\kappaiminusVektor = (\eta_1\, , \ldots \, , \eta_n) $ with integer impacts $\kappaiminus{i}\in\mathbb{N}_0 \, , i \in \{1\, , \ldots \, , n \}$ assigned to each conditional $\satzCL{B_i}{A_i}$ such that $\kappa$ accepts \R\ and is given by: 
  \begin{align}\label{form:def:c-representation}
   \kappa_{ \kappaiminusVektor }(\omega) 
=\sum\limits_{\substack{1 \leq i \leq n\\\omega\models A_i\ol{B}_i}}\kappaiminus{i}
  \end{align}
We will denote the set of all c-representations of $\R$ by $\mathcal{O}(CR(\R )) $.
\end{definition}

As every ranking model of \R, each c-representation
 $\kappa_{ \kappaiminusVektor }$ gives rise to an inference relation according to \eqref{eq_kappa_entailment}. 
While for each consistent \R, the system Z ranking function $\kappa^Z$ is uniquely determined, there may be many different c-representations of \R.
C-inference  \cite{BeierleEichhornKernIsberner2016FoIKSshort,BeierleEichhornKernIsbernerKutsch2018AMAI} is an inference relation taking all c-representations of \R\ into
account.

\begin{definition}[c-inference, \boldmath{$\nmableit_\R^c$} \cite{BeierleEichhornKernIsberner2016FoIKSshort,BeierleEichhornKernIsbernerKutsch2018AMAI}]
\label{def_consequence_relations}
Let \(\R\) be a knowledge base and let $A$, $B$ be formulas.
$B$ is \emph{a (skeptical) c-inference from $A$ in the context of \R}, denoted by $A\nmableit_\R^c B$, iff $A\nmableit^\kappa B$ holds for all c-representations \(\kappa\) for \(\R\).
\end{definition}

In \cite{BeierleEichhornKernIsbernerKutsch2018AMAI}
a modeling of c-representations as solutions of a constraint satisfaction problem \cspR\ 
is given and shown to be sound and complete with respect to the set of
all c-representations of \R.

\begin{definition}[\boldmath{\cspR} \cite{BeierleEichhornKernIsberner2016FoIKSshort,BeierleEichhornKernIsbernerKutsch2018AMAI}]
\label{def_csp_fuer_r}
Let
\(
     \R = \{\satzCL{B_1}{A_1},\ldots,\satzCL{B_n}{A_n}\}
\).
The constraint satisfaction problem for c-representations of \R,
denoted by \cspR, 
on the constraint variables 
    \(\{\kappaiminus{1}, \ldots, \kappaiminus{n}\}\)
ranging over $\mathbb{N}_0$
is given by the conjunction of the constraints,
for all \(i \in \{1,\ldots,n\}\):
\begin{align}
\label{eq_kappaiminus_positive}
&
\kappaiminus{i} \geq 0
\\
\label{eq_kappa_accepts_r_with_kappaiminus}
&
\kappaiminus{i}  >  
   \min_{\omega \models A_i B_i}
           \sum_{\substack{j \neq i \\ \omega \models A_j \ol{B_j}}} \kappaiminus{j} 
    - 
   \min_{\omega \models A_i \notB_i}
           \sum_{\substack{j \neq i \\ \omega \models A_j \ol{B_j}}} \kappaiminus{j} 
\end{align}
\end{definition}

A solution of \cspR\ is an \(n\)-tuple
\(
     (\kappaiminus{1}, \ldots,  \kappaiminus{n}) \in \mathbb{N}_0^n .
\) 
For a constraint satisfaction problem \(\mathit{CSP}\), the set of solutions is denoted by \(\solutionsOf{\mathit{CSP}}\).
Thus, with \solutionsR\ we denote the set of all solutions of \cspR.

\begin{proposition}[soundness and completeness of \cspR\ \cite{BeierleEichhornKernIsberner2016FoIKSshort,BeierleEichhornKernIsbernerKutsch2018AMAI}]\label{prop_modeling_cr}
Let $\mathcal{R} = \{\satzCL{B_1}{A_1}, \ldots, \satzCL{B_n}{A_n}  \}$
  be a knowledge base.
With $\kappa_{\kappaiminusVektor} $ as in~(\ref{form:def:c-representation}), we then have:
\begin{equation}\label{eq_modeling_cr}
\mathcal{O}(CR(\mathcal{R})) = \{ \kappa_{\kappaiminusVektor}  \mid \,  \kappaiminusVektor \in Sol(CR(\mathcal{R}))  \}
\end{equation}
\end{proposition}

\begin{example}[$\mathcal{R}_{\text{bird}}$]\label{exa:bird}
To illustrate the definitions and concepts presented in this paper let us consider an instance of the well known penguin bird example. This example will be a running example and will be continued and extended throughout the paper.
Consider the propositional alphabet $\Sigma = \{p \, , b \, , f \}$ representing whether something is a penguin $(p)$, whether it is a bird $(b)$, or whether it can fly $(f)$.
Thus, the set of worlds is 
\(
\Omega = \{ p \,b \,f \, , p \, b\, \overline{f} \, , p \, \overline{b} \, f \, , p \, \overline{b} \, \overline {f} \, , \overline{p} \, b \, f \, , \overline{p}\, b \,\overline{f} \, , \overline{p}\, \overline{b} \,f   \, ,\overline{p} \, \overline{b} \, \overline{f}   \}
\).
The knowledge base $\mathcal{R}_{\text{bird}} = \{r_1\, , r_2\, , r_3 \, , r_4 \}$ contains the conditionals
\begin{align*}
r_1 = (f | b) \quad & \text{"Birds usually fly"}   , \\
r_2 = (\overline{f} | p) \quad & \text{"Penguins usually do not fly"}  , \\
r_3 = (\overline{f} | bp) \quad & \text{"Penguins which are also birds usually do not fly"}  , \\
r_4 = (b | p) \quad & \text{"Penguins are usually birds"}  .
\end{align*}
For $\R_0=\{\satzCL{f}{b}\}$ and $\R_1=\Rbirds\setminus\R_0$ we have the ordered partitioning $(\R_0,\R_1)$ such that every conditional in $\R_0$ is tolerated by $\R_0\cup\R_1=\Rbirds$ and every conditional in $\R_1$ is tolerated by $\R_1$.
For instance, $\satzCL{f}{b}$ is tolerated by $\Rbirds$ since there is, for example, the world $\ol{p}bf$ with $\ol{p}bf\models bf$ as well as $\ol{p}bf\models(p\Rightarrow\ol{f})\wedge(pb\Rightarrow\ol{f})\wedge(p\Rightarrow b)$. Furthermore $(\mathcal{R}_0 \, , \mathcal{R}_1)$ is indeed the inclusion-maximal partition of $\mathcal{R}$. Therefore, $\R$ is consistent.
An OCF $\kappa$ that accepts $\Rbirds$ is:
\begin{center}
\newcommand{\abstandSSS}{\hspace{3mm}}  
\begin{tabular}{c@{\abstandSSS}c@{\abstandSSS}c@{\abstandSSS}c@{\abstandSSS}c@{\abstandSSS}c@{\abstandSSS}c@{\abstandSSS}c@{\abstandSSS}c}
\toprule
$\omega$&
$\mathit{p\,b\,f}$&
$\mathit{p\,b\,\ol{f}}$&
$\mathit{p\,\ol{b}\,f}$&
$\mathit{p\,\ol{b}\,\ol{f}}$&
$\mathit{\ol{p}\,b\,f}$&
$\mathit{\ol{p}\,b\,\ol{f}}$&
$\mathit{\ol{p}\,\ol{b}\,f}$&
$\mathit{\ol{p}\,\ol{b}\,\ol{f}}$
\\\midrule
$\kappa(\omega)$&$2$&$1$&$2$&$2$&$0$&$1$&$0$&$0$\\
\bottomrule
\end{tabular}
\end{center}
For instance,
we have $\kappa\models\satzCL{f}{b}$ since
$\kappa(bf)=\min\{\kappa(pbf),\, \kappa(\ol{p}bf)\}=\min\{2,0\}=0$ and 
$\kappa(b\ol{f})=\min\{\kappa(pb\ol{f}),\, \kappa(\ol{p}b\ol{f})\}=\min\{1,1\}=1$ and therefore $\kappa(bf)<\kappa(b\ol{f})$.
\end{example}

\section{Preferred Structure on Worlds}\label{sec_sz_structure}

Aiming at developing a nonmonotonic inference relation combining the advantages of system Z and skeptical c-inference, we first introduce the new notion of \namePS on
worlds with respect to a knowledge base \R.
  The idea is to take into account both the tolerance information expressed by the 
ordered partition of \R\ and the structural information which conditionals are falsified.

\begin{definition}[$\falsW^j$, $\falsW$, \namePS \leSZ\ on worlds]\label{def_sz_structure_worlds}
  Consider a consistent knowledge base $\mathcal{R} = \{ r_i=(B_i|A_i) \mid i \in\{1, \ldots, n\} \}$
  with $\fctOP{\R} = \fctOPvector{\R}$.
For $j \in \{0,\ldots,k\}$, 
  $\falsW^j$ and $\falsW$
are the \emph{functions mapping worlds to the set of falsified conditionals}
  from the tolerance partition $\R_j$ and from \R, respectively, given by
\begin{align}
  \falsW^j (\omega) &:= \{ r_i \in \mathcal{R}_ j \mid \omega \models A_i \overline{B_i }  \},\label{eq_mapping_f_j} \\
  \label{eq_mapping_f}
  \falsW(\omega) &:=  \{ r_i \in \mathcal{R} \mid \omega \models A_i \overline{B_i }  \}.
\end{align}
The \emph{\namePS on worlds} is given by the binary relation
$\leSZ \subseteq \Omega \times \Omega$ 
defined by, for any $\omega \, , \omega' \in \Omega$,
\begin{align}
\nonumber
  \omega \leSZ \omega'  \mbox{\, \ \ iff\, \ \ } &\textrm{there exists $m \in \{0\, , \ldots \, , k \}$ such that }\\
  \label{eq_sz_structure}
  &\falsW^i(\omega)  = \falsW^i(\omega') \quad \forall i  \in  \{  m + 1 \, , \ldots \, , k  \}, \,\,\textrm{and}\\
\nonumber
            & \falsW^m(\omega) \subsetneqq  \falsW^m(\omega') \, . 
\end{align}
\end{definition}

Thus, $\omega \leSZ \omega'$ if and only if $\omega $ falsifies strictly less conditionals than $\omega'$ in the partition with the biggest index $m$  where the conditionals falsified by $\omega $ and $\omega'$ differ. 
The \namePS on worlds
will be the basis for defining a new inference relation induced by \R.
Before formally defining this new inference relation and elaborating its properties, we proceed by illustrating
the \namePS on worlds for a knowledge base \R, relating it to c-representations of \R, and proving a set of its properties that
will be useful for investigating the characteristics and properties of
the resulting inference relation.

\begin{example}[\leSZarg{\Rbirds}]\label{ex_sz_order}
Let us determine the \namePS on worlds $\leSZarg{\Rbirds}$  for the knowledge base $\Rbirds$ from Example \ref{exa:bird}  whose verification/falsification behavior is shown in Table~\ref{tab:penguin_verification_and_falsification}.
The inclusion-maximal partition $\fctOP{\Rbirds} = (\R_0, \R_1)$ is given by
\(\mathcal{R}_0  = \{ r_1 = (f | b) \}\)
and
\(\mathcal{R}_1  =  \{ r_2 = (\overline{f} | p) \, , r_3 = (\overline{f} | bp) \, ,   r_4 = (b | p) \}\).
Figure~\ref{fig_preferred_structure_Rbirds} shows the \namePS on worlds $\leSZarg{\Rbirds}$  for the knowledge base $\Rbirds$.
An edge $ \omega \rightarrow \omega' $ between two worlds indicates that $\omega \leSZarg{\Rbirds} \omega'$.
  The full relation $\leSZarg{\Rbirds}$ is obtained from the transitive closure of $\rightarrow $ in Figure~\ref{fig_preferred_structure_Rbirds}.
\end{example}

\begin{table}[tb]
  \begin{center}
\begin{tabular}{l@{\qquad}c@{\quad}c@{\quad}c@{\quad}c@{\quad}c@{\quad}c@{\quad}c@{\quad}c}
  \toprule $\omega$
	&$\mathit{pbf}$&$\mathit{pb\ol{f}}$&$\mathit{p\ol{b}f}$&$\mathit{p\ol{b}\,\ol{f}}$&$\mathit{\ol{p}bf}$&$\mathit{\ol{p}b\ol{f}}$&$\mathit{\ol{p}\ol{b}f}$&$\mathit{\ol{p}\ol{b}\,\ol{f}}$\\
  \midrule
  $r_1=(f|b)$& \verify  &  \falsif  &  \neutra  &  \neutra  & \verify  &  \falsif  &  \neutra &  \neutra \\
  $r_2=(\ol{f}|p)$& \falsif  & \verify  & \falsif  & \verify  &  \neutra &   \neutra &  \neutra &  \neutra \\
  $r_3=(\ol{f}|pb)$& \falsif  & \verify  &  \neutra &  \neutra &  \neutra &  \neutra &  \neutra &  \neutra \\
  $r_4=(b|p)$& \verify  & \verify  & \falsif  & \falsif  &  \neutra &  \neutra &  \neutra &  \neutra \\
  \midrule
$\kappa^Z(\omega)$&$2$&$1$&$2$&$2$&$0$&$1$&$0$&$0$\\
\bottomrule
 \end{tabular}
 \end{center}
 \caption{Verification/falsification behavior of the knowledge base \Rbirds; (\verify) indicates verification, (\falsif) falsification, and (\neutra) non-applicability. The OCF $\kappa^Z$ is the ranking function obtained from $\Rbirds$ using system~Z.}
 \label{tab:penguin_verification_and_falsification}
\end{table}

\begin{figure}\centering
\begin{tikzpicture}
  \node (level4spalte3) at (0,0) {$p\ol{b}f$};
  \node [left of=level4spalte3] (level4spalte2) {};
  \node [left of=level4spalte2] (level4spalte1) {};
  \node [right of=level4spalte3] (level4spalte4) {};
  \node [right of=level4spalte4] (level4spalte5) {};
\node [below of=level4spalte1] (level3spalte1)  {$pbf$};
  \node [below of=level4spalte2] (level3spalte2)  {};
  \node [below of=level4spalte3] (level3spalte3)  {$p\ol{b}\,\ol{f}$};
  \node [below of=level4spalte4] (level3spalte4)  {};
  \node [below of=level4spalte5] (level3spalte5) {};
\node [below of=level3spalte1] (level2spalte1)  {};
  \node [below of=level3spalte2] (level2spalte2)  {$pb\ol{f}$};
  \node [below of=level3spalte3] (level2spalte3)  {};
  \node [below of=level3spalte4] (level2spalte4)  {$\ol{p}b\ol{f}$};
  \node [below of=level3spalte5] (level2spalte5)  {};
\node [below of=level2spalte1] (level1spalte1)  {$\ol{p}bf$};
  \node [below of=level2spalte2] (level1spalte2)  {};
  \node [below of=level2spalte3] (level1spalte3)  {$\ol{p}\ol{b}f$};
  \node [below of=level2spalte4] (level1spalte4)  {};
  \node [below of=level2spalte5] (level1spalte5) {$\ol{p}\ol{b} \, \ol{f}$};
\draw [black,  thick, ->] (level1spalte1) -- (level2spalte2);
  \draw [black,  thick, ->] (level1spalte1) -- (level2spalte4);
  \draw [black,  thick, ->] (level1spalte3) -- (level2spalte2);
  \draw [black,  thick, ->] (level1spalte3) -- (level2spalte4);
  \draw [black,  thick, ->] (level1spalte5) -- (level2spalte4);
  \draw [black,  thick, ->] (level1spalte5) -- (level2spalte2);
\draw [black,  thick, ->] (level2spalte2) -- (level3spalte1);
  \draw [black,  thick, ->] (level2spalte2) -- (level3spalte3);
  \draw [black,  thick, ->] (level2spalte4) -- (level3spalte1);
  \draw [black,  thick, ->] (level2spalte4) -- (level3spalte3);
\draw [black,  thick, ->] (level3spalte3) -- (level4spalte3);
\end{tikzpicture}
\caption{The \namePS  relation  \leSZarg{\Rbirds} on worlds for the knowledge base \Rbirds.}
    \label{fig_preferred_structure_Rbirds}
\end{figure}

The following proposition can be seen as a generalization of a result from \cite{BeierleKutsch2019AppliedIntelligence}.
It extends \cite[Proposition 15]{BeierleKutsch2019AppliedIntelligence} to the relation $\leSZ$ and to arbitrary knowledge bases, not just knowledge bases  only consisting of conditional facts  as in  \cite[Proposition 15]{BeierleKutsch2019AppliedIntelligence}.
It tells us that the set of c-representations is rich enough to guarantee the existence of a particular c-representation $ \kappa_{ \kappaiminusVektor } \in \mathcal{O}(\cspR)$ fulfilling
the ordering constraints given in the proposition.

\begin{proposition}\label{lemma_technical_sz}
Let $\mathcal{R} = \{ r_i = (B_i|A_i) \mid i  =1 \, , \ldots \, , n  \}$ be a consistent knowledge base, let $\omega' \in \Omega  $ and let $\Omega_V \subseteq \Omega $. Assume that  $ \omega \not \leSZ \omega ' $ for all $\omega \in \Omega_V$.
Then there exists a solution $ \kappaiminusVektor \in \solutionsR$ and thus a c-representation, 
 $\kappa_{\kappaiminusVektor} \in \mathcal{O}(CR(\mathcal{R}))$ 
such that, for all $\omega \in \Omega_V$, we have:
\begin{equation}\label{eq_technical_sq}
\kappa_{ \kappaiminusVektor }(\omega') \leq \kappa_{ \kappaiminusVektor }(\omega)
\end{equation}
\end{proposition}

\begin{proof} \emph{(Sketch)} Due to lack of space, we give a sketch of the proof.
The claim follows by combining the following two statements:
\begin{itemize}
\item[(i)] If $ \eta_i \in \mathbb{N}  \, , i \in \{1\, , \ldots \, , n \}  $, satisfy
\begin{equation}\label{Un.E11}
\eta _i > \sum_{ \substack{ j \in \{1 , \ldots  , n \}  \\ r_j \in \bigcup_{l=0}^{m-1} \mathcal{R}_l}  } \eta_j
\end{equation}
for all $i \in \{ 1  \, , \ldots \, , n \} $ where $m=m(i)\in \{0\, , \ldots \, , k \}$ with $r_i \in \mathcal{R}_m $ then $\kappaiminusVektor = (\eta_1 \, , \ldots \, , \eta_n )$ is a solution of $CR(\mathcal{R})$ and so $\kappa_{\kappaiminusVektor} $ defined as in~(\ref{form:def:c-representation}) is a c-representation of $\R$.
\item[(ii)]
Because of  $ \omega \not \leSZ \omega ' $ for all $\omega \in \Omega_V$ we can choose $\kappaiminusVektor = (\eta_1 \, , \ldots \, , \eta_n ) $
satisfying~(\ref{Un.E11})
  such that
$\kappa_{\kappaiminusVektor} $ defined as in~(\ref{form:def:c-representation}) satisfies ~(\ref{eq_technical_sq}) for all $\omega \in \Omega_V$.
\end{itemize}
A complete proof that (i), (ii)
hold is given in the full version of this paper.
\MEINqed
\end{proof}

The rest of this section is dedicated to the investigation of further properties of the relation $\leSZ$.  Let us start with a lemma that tells us that worlds falsifying the same sets of conditionals are equivalent with respect to $\leSZ$.

\begin{lemma}\label{lemma_cond_invariance}
Let $ \R = \{ (B_i|A_i) \mid i =1, \ldots, n \}$ be a knowledge base, and
let $\omega_1 \, , \omega_2 \in \Omega $ falsify the same sets of conditionals, i.e., for all $i \in \{1\, , \ldots \, , n \}$, we have
\(\omega_1 \models A_i \overline{B}_i\)
iff
\(\omega_2 \models A_i \overline{B}_i\).
Then $\omega_1 \,, \omega_2 $ behave exactly the same way with respect to $\mathcal{R}$,
i.e., for all $\omega \in \Omega$, the following equivalences hold:
\begin{align*}
\omega \leSZ \omega_1  \iff \omega  \leSZ \omega_2 \, , \\
\omega_1 \leSZ \omega  \iff \omega_2  \leSZ \omega \, .
\end{align*} 
\end{lemma}

\begin{proof}
The claim follows from $\falsW^i(\omega_1) = \falsW^i(\omega_2)$ for all $i \in \{0 \, , \ldots \, , k \}$.
\MEINqed
\end{proof}

In general, the relation $\leSZ$ cannot be obtained from a ranking function. 
\begin{lemma}\label{lemma_no_ranking_inference}
There exists a knowledge base $\R$ such that there is no ranking function $\kappa : \Omega \to \mathbb{N}_0^{\infty} $
  with \(\omega_1 \leSZ \omega_2\) iff \(\kappa(\omega_1) < \kappa(\omega_2)\).  
\end{lemma}

\begin{proof}
The proof is by contradiction. Assume there is a ranking function $\kappa: \Omega \to \mathbb{N}_0^{\infty} $ with \(\omega_1 \leSZ \omega_2\) iff \(\kappa(\omega_1) < \kappa(\omega_2)\) for \Rbirds.
For $\leSZ$ (cf.\ Figure~\ref{fig_preferred_structure_Rbirds}) we have
$p \, b \, f  \not \leSZarg{\Rbirds} p \, \overline{b} \, f $ and $p \, \overline{b} \, f  \not \leSZarg{\Rbirds} p \, b \, f $
and furthermore
$p \, b \, f  \not \leSZarg{\Rbirds} p \, \overline{b} \, \overline{f} $ and $ p \, \overline{b} \, \ol{f}  \not \leSZarg{\Rbirds} p \, b \, f $.
Therefore, we obtain 
\(
\kappa(p \, b \, f ) = \kappa( p \, \overline{b} \, f)
\)
and
\(
\kappa(p \, b \, f  =  \kappa(p \, \overline{b} \, \overline{f})
\).
Thus, $\kappa( p \, \overline{b} \, f) = \kappa(p \, \overline{b} \, \overline{f})$ which is a contradiction to $ p \, \overline{b} \, \overline{f} \leSZarg{\Rbirds} p  \, \overline{b} \, f$.
\MEINqed
\end{proof}

Let us end this subsection by proving that $\leSZ$ defines a strict partial order.
\begin{lemma}\label{lemma_inference_transitive}
The relation $\leSZ$ is irreflexive, antisymmetric and transitive, meaning that  $\leSZ $ is a strict partial order.
\end{lemma}

\begin{proof}
  Condition \eqref{eq_sz_structure}
  immediately yields that $\leSZ$ is irreflexive and antisymmetric. It remains to show that  $\leSZ$ is transitive.
Define
\(a := \max \{ i \in \{0, \ldots, k \} \mid  \falsW^i(\omega_1) \neq \falsW^i(\omega_2)  \}\)
and
\(b := \max \{ i \in \{0, \ldots, k \} \mid  \falsW^i(\omega_2) \neq \falsW^i(\omega_3)  \}\).
Then $\omega_1 \leSZ \omega_2 $ and $\omega_2 \leSZ \omega_3$ is equivalent to
\(\falsW^a(\omega_1) \subsetneqq  \falsW^a(\omega_2)\)
and
\(\falsW^b(\omega_2) \subsetneqq  \falsW^b(\omega_3)\).

If $ a = b $ then 
\(
\falsW^a(\omega_1) \subsetneqq  \falsW^a(\omega_3)\) and \(a = \max \{ i \in \{0, \ldots,   k \} \mid \falsW^i(\omega_1) \neq \falsW^i(\omega_3) \}
\)
and so $ \omega_1 \leSZ \omega_3$.
If $ a < b $ then 
\(
 \falsW^b(\omega_1) \subsetneqq  \falsW^b(\omega_3)\) and \(b = \max \{ i \in \{0, \ldots, k \} \mid  \falsW^i(\omega_1) \neq \falsW^i(\omega_3)  \}
\)
and so $ \omega_1 \leSZ \omega_3$.
If $ a > b $ then $\falsW^i(\omega_2) = \falsW^i(\omega_3)$ for all $i \in \{ b+1, \ldots, k \}$ and $ b+1 \leq a \leq k$; therefore
\(
\falsW^a(\omega_1) \subsetneqq  \falsW^a(\omega_3)\) and \(a = \max \{ i \in \{0, \ldots, k \} \mid \falsW^i(\omega_1) \neq \falsW^i(\omega_3)  \}
\)
and so $ \omega_1 \leSZ \omega_3$.
\MEINqed
\end{proof}

\section{\systemSZgross}\label{sec_system_sz}

The \namePS\ \leSZ\ on worlds for a knowledge base \R\ is defined using both the tolerance information provided by the inclusion-maximal ordered partition $\fctOP{\R}$ and information on the set of falsified conditionals. Inference based on \leSZ\ is called \systemSZ inference
and is defined as follows.

\begin{definition}[\systemSZ, \inferSZ]\label{def_sz_inference}
  Let \R\ be a
knowledge base and $A, B$ be formulas.
  Then $B$ is a  \emph{\systemSZ inference from $A$ (in the context of \R)}, denoted
\begin{equation}\label{eq_sz_inference}
A \inferSZ B  \mbox{\, \ \ iff\, \ \ } \forall \omega' \in \Omega_{A\overline{B}} \; \exists \omega \in \Omega_{A B} \; \omega \leSZ \omega ' \, .
\end{equation}
\end{definition}
A consequence of this definition is that \systemSZ inference is as tractable as system~Z because the \namePS on worlds is obtained directly from the ordered partition of \R\ and the verification/falsification behavior of \R.
We apply the definition of \systemSZ to our running example.

\begin{example}[\Rbirds, cont.]\label{exa_bird_systemSZ}
Consider again \Rbirds\ from Example \ref{exa:bird}.
Let us show that for $A=b \,  f $ and $B =  \ol {p}$ we have $A \inferSZArg{\Rbirds} B$, i.e., that flying birds are usually not penguins.
Due to
\(\falsW( b \, f \, \ol{p} ) = \varnothing\) and \(\falsW(b \, f \, p) = \{ r_2, r _3 \}\)
(see Table~\ref{tab:penguin_verification_and_falsification})
it follows that $ b \,  f \,  \ol{p} \leSZarg{\Rbirds} b \,  f \, p$. Therefore, since $\Omega_{A B} = \{ b \,  f \,   \ol {p} \} $ and $\Omega_{A \ol {B}} = \{ b \,  f \,   p \} $, from (\ref{eq_sz_inference}) it follows that indeed $ b \, f \inferSZArg{\Rbirds} \ol {p} $. 
\end{example}

Note that  $b \, f \nmableit_{\Rbirds}^c \ol {p}$, i.e., this inference is also a skeptical c-inference (cf.\ \cite[Example 5]{BeierleEichhornKernIsbernerKutsch2018AMAI}).
Therefore, Example~\ref{exa_bird_systemSZ} presents a c-inference that is also a \systemSZ inference.
The following proposition tells us that $A \nmableit _{\mathcal{R}}^{c} B$ always implies $A \inferSZ B$.
\begin{proposition}[\systemSZ captures c-inference]\label{prop_extends_c_inference}
Let $\mathcal{R} $ be a consistent knowledge base.
Then we have for all formulas $ A \, , B \in \mathcal{L} $
\begin{equation}\label{eq_extends_c_inference}
A \nmableit_{\mathcal{R}}^{c} B \quad \implies A \inferSZ B  \, .
\end{equation} 
\end{proposition}

\begin{proof}
  The proof of~(\ref{eq_extends_c_inference}) is by contraposition.
  Assume $A \notinferSZ B$ and thus
\begin{equation}\label{Nec.E3}
\exists \omega' \in \Omega_{A\overline{B}} \; \forall \omega \in \Omega_{A B} \; \omega  \not \leSZ \omega ' \, .
\end{equation}
Our goal is to show   $A \notnmableit _{\mathcal{R}}^{c} B$.
Let us fix $\omega' \in \Omega_{A \overline{B} } $ such that~(\ref{Nec.E3}) holds. Let us define $\Omega_V := \Omega_{A B} $. 
Then $ \omega \not \leSZ \omega ' $  for all $\omega \in \Omega_V $. Due to Lemma~\ref{lemma_technical_sz} there exists a c-representation  $\kappa_{\kappaiminusVektor} \in \mathcal{O}(CR(\mathcal{R}))$
 such that
$ \kappa_{ \kappaiminusVektor }(\omega') \leq \kappa_{ \kappaiminusVektor }(\omega) $
 for all $ \omega \in \Omega_{A B} $. This means that
$A \notnmableit _{\mathcal{R}}^{\kappa_{\kappaiminusVektor}} B $ 
and so indeed $A \notnmableit _{\mathcal{R}}^{c} B$.
\MEINqed
\end{proof}

Furthermore, every system Z inference is also a \systemSZ inference.

\begin{proposition}[\systemSZ captures system Z]\label{prop_extends_system_z_inference}
Let $\mathcal{R} $ be a consistent knowledge base.
Then we have for all formulas $ A \, , B \in \mathcal{L} $
\begin{equation}\label{eq_extends_system_z_inference}
A \nmableit_{\R}^{Z} B \quad \implies A \inferSZ B  \, .
\end{equation} 
\end{proposition}

\begin{proof}
Inspecting
(\ref{eq_system_z}) and (\ref{eq_sz_structure}) and given any worlds $\omega , \omega' \in  \Omega$, 
   we conclude that 
\(\kappa^Z(\omega) < \kappa^Z(\omega')\) implies \(\omega \leSZ \omega'\). 
Therefore,
comparing~(\ref{eq_kappa_entailment_equiv}), applied to the ranking function $\kappa^{Z}$, with~(\ref{eq_sz_inference}), shows that~(\ref{eq_extends_system_z_inference}) is fulfilled.
\MEINqed
\end{proof}

In \cite{KernIsbernerEichhorn2012c},
a preference relation on worlds is defined that is based on structural
information by preferring a world $\omega$ to a world $\omega'$  if $\omega$ falsifies fewer conditionals than  $\omega'$ and  $\omega'$ falsifies at least all conditionals falsified by $\omega$.
Using this preference relation, the following entailment relation along the scheme
as given by \eqref{eq_kappa_entailment_equiv} is obtained;
we present the definition from  \cite{KernIsbernerEichhorn2012c} in a slightly modified form adapted to our notion $\falsW(\omega)$ for the set of conditionals from \R\ falsified by~$\omega$.

\begin{definition}[$\sigma_{\R}$-structural inference~\cite{KernIsbernerEichhorn2012c}]
Let
$\R = \{r_1,\ldots,r_n\}$ 
with $r_i = \satzCL{B_i}{A_i}$ for $i = 1,\ldots,n$
  be a knowledge base, $A,B$ formulas, and let $<_{\R}^{\sigma}$ be the relation on worlds given by
\(\omega <_{\R}^{\sigma} \omega'\) iff
\(\falsW(\omega)  \subsetneqq \falsW(\omega')\).
Then $B$ can be
structurally inferred, or $\sigma_{\R}$-inferred, from $A$, written as
\begin{equation}\label{eq_structural_inference}
A \nmableit_{\R}^{\sigma} B \mbox{\, \ \ iff\, \ \ } \forall \omega' \in \Omega_{A\overline{B}} \; \exists \omega \in \Omega_{A B} \; \omega <_{\R}^{\sigma} \omega ' \, .
\end{equation}
\end{definition}

We can show that every $\sigma_{\R}$-structural inference is also a \systemSZ inference.

\begin{proposition}[\systemSZ captures $\sigma_{\R}$-structural inference]\label{prop_extends_structural_inference}
Let $\mathcal{R} $ be a consistent knowledge base.
Then we have for all formulas $ A \, , B \in \mathcal{L} $
\begin{equation}\label{eq_extends_structural_inference}
A \nmableit_{\R}^{\sigma} B \quad \implies A \inferSZ B  \, .
\end{equation} 
\end{proposition}

\begin{proof}
  Inspecting (\ref{eq_sz_structure}) and the definition of $<_{\R}^{\sigma}$,
we conclude that 
   \(\omega <_{\R}^{\sigma} \omega'\) implies \(\omega \leSZ \omega'\)
   for  all $\omega , \omega' \in  \Omega$.
Combining (\ref{eq_sz_structure}) and (\ref{eq_structural_inference})
yields~(\ref{eq_extends_structural_inference}).
\MEINqed
\end{proof}

The following proposition summarizes Propositions \ref{prop_extends_c_inference}, \ref{prop_extends_system_z_inference}, and
\ref{prop_extends_structural_inference} and
shows aditionally that \systemSZ strictly extends skeptical c-inference, system Z, and structural inference
by licensing more entailments than each of these three inference modes.

\begin{proposition}[\systemSZ]\label{prop_system_W_captures_all}
 For every consistent knowledge base \R\
\begin{align}
\label{eq_summary_includes_all}
\nmableit_{\R}^{c} \, \subseteq \inferSZ,
\,\,
\nmableit_{\R}^{Z} \, \subseteq \inferSZ\,\,
\textrm{ and \ }
\nmableit_{\R}^{\sigma} \, \subseteq \inferSZ. 
\end{align}
Furthermore, all inclusions in \eqref{eq_summary_includes_all} are strict, i.e., there are $\R_1 \, , \R_2 \, , \R_3$ with:
\begin{align}
\label{eq_summary_includes_strictly_c_rep}
\nmableit_{\R_1}^{c} \, \subsetneqq \inferSZarg{\R_1}
\\
\label{eq_summary_includes_strictly_z}
\nmableit_{\R_2}^{Z} \, \subsetneqq  \inferSZarg{\R_2} 
\\
\label{eq_strictly_includes_strictly_sigma}
\nmableit_{\R_3}^{\sigma} \, \subsetneqq \inferSZarg{\R_3} 
\end{align}
\end{proposition}

\begin{proof}
  The inclusions in \eqref{eq_summary_includes_all} are shown in Propositions \ref{prop_extends_c_inference}, \ref{prop_extends_system_z_inference}, and~\ref{prop_extends_structural_inference}.
  Thus, we are left to show that the inclusions in \eqref{eq_summary_includes_strictly_c_rep} -- \eqref{eq_strictly_includes_strictly_sigma} are strict.
\begin{enumerate}

\item For proving the strictness part of (\ref{eq_summary_includes_strictly_c_rep}), consider the knowledge base $\R^* = \{(b|a) , ( b \, c| a) \}$ whose verification/falsification behavior is given by Table~\ref{tab:r_star_verification_and_falsification}.
First, due to 
\(
\falsW(a \,  b \, \overline{c} ) = \{ ( b \, c| a) \} \subsetneqq \{(b|a) , ( b \, c| a) \} =  \falsW( a \, \overline{ b} \, \overline{c} )  
\),
we obtain $  a \, \overline{c}  \inferSZArg{\R^*}  b$.
Making use of the verification/falsification behavior stated  in Table~\ref{tab:r_star_verification_and_falsification},  for $\cspRArg{\R^*}$ we obtain 
\(\eta_1  >  -\eta_2\) and \(\eta_2 > 0 \). Now
consider the solution vector $\kappaiminusVektor = (\eta_1 , \eta_2 ) = (0,1) $.
For the associated c-representation $\kappa_{\kappaiminusVektor}$ 
(see Table~\ref{tab:r_star_verification_and_falsification}) we then obtain
\(
\kappa_{\kappaiminusVektor} ( a \, b \, \ol {c} ) =  \eta_2  =   \eta_1 + \eta_2 = \kappa_{\kappaiminusVektor} ( a \,  \ol{b} \, \ol {c} )
\)
and thus $ a \, \overline{c} \notnmableit_ {\R^*}^c b $.

\item For proving the strictness part of (\ref{eq_summary_includes_strictly_z}), consider 
the knowledge base $\Rbirds$  from Example~\ref{exa:bird}. Let us show that for $A=p \,  \ol {b} $ and $B =  \ol {f}$ we have $A \inferSZarg{\Rbirds} B$, i.e., that penguins which are no bird usually do not fly.
According to  Example~\ref{ex_sz_order}, we have $ p \, \ol{b} \, \ol{f} \leSZarg{\Rbirds} p \, \ol{b} \, f$. Therefore, since $\Omega_{A B} = \{ p \, \ol{b} \,   \ol {f} \} $ and $\Omega_{A \ol {B}} = \{ p \, \ol{b} \,  f \} $ it follows from
  (\ref{eq_sz_inference}) that indeed $ p \, \ol{b} \inferSZArg{\Rbirds} \ol {f} $. 
Looking at Table~\ref{tab:penguin_verification_and_falsification}, we observe
\(
\kappa^Z(A B) = 2 = \kappa^Z(A \ol {B})
\)
and thus
$ p \, \ol{b} \notnmableit_{\Rbirds}^Z \ol {f} $.

\item  For proving the strictness part of (\ref{eq_strictly_includes_strictly_sigma}),  consider again 
  $\Rbirds$ with $\fctOP{\Rbirds} = (\R_0, \R_1)$ where
\(
\mathcal{R}_0  = \{ (f | b) \}\) and \(\mathcal{R}_1  =  \{ (\overline{f} | p), \, (\overline{f} | bp), \, (b | p) \}
\)
(cf.\ Example~\ref{ex_sz_order}).
For  $\omega = p \, b \, f $, we get (cf.\ Table~\ref{tab:penguin_verification_and_falsification}) that
$\falsW(\omega) = \{ \satzCL{\overline{f}}{p}, \satzCL{\overline{f}}{pb} \}$, 
\(
\falsW(p \, b \, \overline{f} ) = \{ (f|b)  \}\) and  \(\falsW(p \, \overline {b} \, \overline{f})  = \{( b|p ) \}
\).
Thus, there is no world $\omega' \in \Omega $ with $\omega' \models p \, \overline{f} $ and $\omega' <_{\Rbirds}^{\sigma} \omega$ (which is equivalent to $\falsW(\omega') \subsetneqq \falsW(\omega)$). Therefore,  $p   \notnmableit_{\Rbirds}^{\sigma} \overline{f}$.
To show $ p \inferSZ \overline{f}$ fix any $\omega \in \Omega$ with $ \omega \models   p  \, f $. Then $ (\overline {f} | p) \in \falsW(\omega) $ where $ (\overline {f} | p) \in R_1$.
For  $\omega' = p \, b \, \overline{f}$ we have
$ \omega' \leSZ \omega$   due to $\falsW(\omega') = \{(f|b) \}$ where $(f|b) \in \R_0$. Thus, indeed   $ p \inferSZ \overline{f}$.
\MEINqed
\end{enumerate}
\end{proof}

\begin{table}[tb]
\begin{center}
 \begin{tabular}{c@{\qquad}c@{\quad}c@{\quad}c@{\quad}c@{\quad}c@{\quad}c@{\quad}c@{\quad}c}
  \toprule $\omega$
	&$\mathit{ a b c}$&$\mathit{ a b \ol{c}}$&$\mathit{ a \ol{b} c}$&$\mathit{a\ol{b}\,\ol{c}}$&$\mathit{\ol{a}bc}$&$\mathit{\ol{a}b\ol{c}}$&$\mathit{\ol{a}\ol{b}c}$&$\mathit{\ol{a}\ol{b}\,\ol{c}}$\\
  \midrule
  $r_1=(b|a)$&$\verify$& $\verify$& $\falsif$ & $\falsif$ &$-$& $-$& $-$& $-$\\
  $r_2=( b c |a)$&$\verify$&$\falsif$&$\falsif$&$\falsif$& $-$&  $-$& $-$& $-$\\
  \midrule
$\induzierteOCF{\kappaiminusVektor}(\omega)$&$0$&$\kappaiminus{2}$&$\kappaiminus{1}+\kappaiminus{2}$&$\kappaiminus{1}+\kappaiminus{2}$&$0$&$0$&$0$&$0$\\
\bottomrule
 \end{tabular}
\end{center} 
 \caption{Verification/falsification behavior and (generic) c-representation  of the knowledge base $\R^{*}$ in the proof of Proposition~\ref{prop_system_W_captures_all};  (\verify) indicates verification, (\falsif) falsification and (\neutra) non-applicability.}
 \label{tab:r_star_verification_and_falsification}
\end{table}

After comparing \systemSZ with other established
inference methods let us deal with further of its properties. 
\begin{proposition}
In general, \systemSZ inference cannot be obtained from a ranking function, i.e., there exists a knowledge base $\R$ such that there is no ranking function $\kappa : \Omega \to \mathbb{N}_0^{\infty} $ with $\inferSZ = \nmableit^\kappa$.
\end{proposition}
\begin{proof}
  This follows immediately from Lemma~\ref{lemma_no_ranking_inference}.
  \MEINqed
\end{proof}

Nonmonotonic inference relations are usually evaluated by means of properties.
In particular, the axiom system~P~\cite{Adams1975,KrausLehmannMagidor90} provides an important standard for plausible, nonmonotonic inferences.

\begin{proposition}\label{prop_satisfies_system_p}
\systemSZgross inference satisfies System P.
\end{proposition}
\begin{proof}
According to Lemma~\ref{lemma_inference_transitive}, \leSZ\ is a strict transitive relation. Furthermore, since $\Omega $ is finite, the triple $\mathcal{M}^{w}(\R)  =  [ \Omega , \models, \leSZ ]$ is a  stoppered classical preferential model \cite{Makinson94}. Thus, the definition of \systemSZ given by \eqref{eq_sz_inference} in Definition~\ref{def_sz_inference} ensures that \systemSZ inference is a preferential inference, hence satisfying system P (cf.\ \cite{KrausLehmannMagidor90,Makinson94}).
\MEINqed
\end{proof}

An inference relation suffers from the \emph{Drowning Problem} \cite{Pearl1990,BenferhatCayrolDuboisLangPrade1993} if it does not allow to infer properties of a superclass for a subclass that is exceptional with respect to another property 
because the respective conditional is ``drowned'' by others.
E.g., penguins are exceptional birds with respect to flying but not with respect to having wings.
So we would reasonably expect that penguins have wings.

\begin{example}[\RbirdsStern \cite{BeierleEichhornKernIsbernerKutsch2018AMAI}]\label{exa_nodrown}
	We extend the alphabet $\Sigma=\{\mathit{p,b,f}\}$ of our running example knowledge base \Rbirds\ from Example~\ref{exa:bird} 
with the variable $w$ for \emph{having wings}, the variable $a$ for \emph{being airborne}, and the variable $r$ for \emph{being red}, obtaining
the alphabet $\Sigma^\ast=\{\mathit{p,b,f,w,a,r}\}$.
	We use the knowledge base 
\[
	\RbirdsStern=\big\{
	\satzCL{f}{b},\,\,
	\satzCL{\ol{f}}{p},\,\,
	\satzCL{b}{p},\,\,
	\satzCL{w}{b},\,\,
	\satzCL{a}{f}\big\}
\]
where the conditional $\satzCL{w}{b}$ encodes the rule that birds usually have wings, and the conditional $\satzCL{a}{f}$ encodes the rule that flying things are usually airborne; the other three conditionals \(\satzCL{f}{b}\), \(\satzCL{\ol{f}}{p}\), \(\satzCL{b}{p}\) are the same as in \Rbirds.
\end{example}

The Drowning Problem distinguishes between inference relations that allow for subclass inheritance only for non-exceptional subclasses (like system~Z inference) and inference relations that allow for subclass inheritance for exceptional subclasses (like skeptical c-inference~\cite[Observation 1]{BeierleEichhornKernIsbernerKutsch2018AMAI} and inference with minimal c-representations, cf.~\cite{KernIsbernerEichhorn2012c,ThornEichhornKernIsbernerSchurz2015b}).
As an illustration for the drowning problem, consider \RbirdsStern\ from Example~\ref{exa_nodrown}. Here, we have
\(
\kappa^Z(p \, \ol{w}) = 1 = \kappa^Z (p \, w)
\),
and consequently  $p \notnmableit_{\RbirdsStern}^{Z} w$ (cf.\ \cite[Example 9]{BeierleEichhornKernIsbernerKutsch2018AMAI}),
illustrating that system Z suffers from the drowning problem.
In contrast, the following observation shows that \systemSZ licenses the inference that penguins usually have wings and thus avoids this drowning phenomenon.

\begin{Observation}\label{obs_no_drowing_sz}
\systemSZgross inference does not suffer from the drowning problem in Example~\ref{exa_nodrown},  i.e., we have $p  \inferSZarg{\RbirdsStern} w$.
\end{Observation}

\begin{proof}
The inclusion-maximal partition  $ \fctOP{\RbirdsStern} =(\R_0,\R_1) $   of $\RbirdsStern  $  in Example~\ref{exa_nodrown} is given by
\(\R_0=\{ (f|b), (w|b), (a|f)\}\) and
\(\R_1=\{ (\overline{f} |p) , (b|p)\}\).

Consider $\omega \in \Omega$ with $\omega \models p \, \overline{w}$.  Choose an arbitrary $\omega' \in \Omega$ with  $\omega ' \models  p \, b \, 
\overline{f} w $.   We will  show $\omega ' \leSZarg{\RbirdsStern} \omega$.
Obviously, $\omega' $ falsifies only the conditional $(f|b)$ which is in $\mathcal{R}_0$, written as a formula $\falsW(\omega) = \{ (f|b) \}$. Since $\omega \models  p \overline{w}$, we can distinguish the following two cases:
\begin{itemize}
\item[(i)] If $\omega \models  p \, \overline{w} \, f $ then the conditional $ (\overline{f} |p) $  from $\mathcal{R}_1$ is falsified.
\item[(ii)] If $\omega \models  p \, \overline{w} \,  \overline{f} $ then we can again distinguish two cases:
\begin{itemize}
\item[(a)] If $\omega \models p \, \overline{w} \, \overline{f} \, \overline{b}$ then $(b|p)$ from $\mathcal{R}_1$ is falsified.
\item[(b)] If  $\omega \models p \, \overline{w} \, \overline{f} \, b$ then at least $(f|b), (w|b)$ (both from $\mathcal{R}_0$) are falsified.
\end{itemize} 
\end{itemize}
Due to \eqref{eq_sz_structure}, we thus get $\omega ' \leSZarg{\RbirdsStern} \omega$ in every case, implying  $p  \inferSZarg{\RbirdsStern} w$.  \MEINqed
\end{proof}

\section{Conclusions and Future Work}\label{sec_conclusions}

In this paper, we introduced \systemSZ and its underlying \namePS of worlds.
\systemSZgross inference captures both System~Z inference and skeptical c-inference and exhibits desirable properties. For instance, in contrast to system~Z, it avoids the drowning problem. In contrast to skeptical c-inference, it does not require to solve a complex constraint satisfaction problem, but is as tractable as system~Z because the \namePS on worlds is obtained directly from the ordered partition of \R\ and the verification/falsification behavior of \R.
In future work, we will empirically evalute \systemSZ with the reasoning platform InfOCF~\cite{BeierleEichhornKutsch2017KIzeitschrift} and investigate further inference properties of it.

\bibliographystyle{splncs03}

\end{document}